\documentclass{article}

\usepackage{microtype}
\usepackage{graphicx}
\usepackage{subcaption}
\usepackage{booktabs}
\usepackage{hyperref}

\usepackage[preprint]{icml2026}

\usepackage{amsmath}
\usepackage{amssymb}
\usepackage{mathtools}
\usepackage{amsthm}
\usepackage{algorithm}
\usepackage{algorithmic}
\usepackage{threeparttable}

\newcommand{\cX}{\mathcal{X}}
\newcommand{\R}{\mathbb{R}}
\newcommand{\E}{\mathbb{E}}
\newcommand{\Vol}{\mathrm{Vol}}
\newcommand{\Prob}{\mathbb{P}}
\newcommand{\tildeO}{\tilde{O}}
\newcommand{\UCB}[1]{\mathrm{UCB}_{#1}}
\newcommand{\LCB}[1]{\mathrm{LCB}_{#1}}
\newcommand{\cE}{\mathcal{E}}

\newcommand{\cT}{\mathcal{T}}
\DeclareMathOperator*{\argmax}{arg\,max}

\usepackage[capitalize,noabbrev]{cleveref}
\usepackage{tikz}
\usetikzlibrary{patterns,decorations.pathreplacing,calc}

\theoremstyle{plain}
\newtheorem{theorem}{Theorem}[section]
\newtheorem{proposition}[theorem]{Proposition}
\newtheorem{lemma}[theorem]{Lemma}

\theoremstyle{definition}

\newtheorem{assumption}[theorem]{Assumption}
\theoremstyle{remark}
\newtheorem{remark}[theorem]{Remark}

\icmltitlerunning{Certificate-Guided Pruning for Stochastic Lipschitz Optimization}

\begin{document}

\twocolumn[
  \icmltitle{Certificate-Guided Pruning for Stochastic Lipschitz Optimization}
  \icmlsetsymbol{equal}{*}
  \begin{icmlauthorlist}
    \icmlauthor{Ibne Farabi Shihab}{equal,cs}
    \icmlauthor{Sanjeda Akter}{equal,cs}
    \icmlauthor{Anuj Sharma}{cee}
  \end{icmlauthorlist}
\icmlaffiliation{cs}{Department of Computer Science, Iowa State University, Ames, Iowa, USA}
\icmlaffiliation{cee}{Department of Civil, Construction \& Environmental Engineering, Iowa State University, Ames, Iowa, USA}
\icmlcorrespondingauthor{Ibne Farabi Shihab}{ishihab@iastate.edu}
\icmlkeywords{Black-box Optimization, Lipschitz Bandits, Safe Optimization, Sample Complexity, Hyperparameter Tuning}
  \vskip 0.3in
]

\printAffiliationsAndNotice{\icmlEqualContribution}

\begin{abstract}
We study black-box optimization of Lipschitz functions under noisy evaluations. Existing adaptive discretization methods implicitly avoid suboptimal regions but do not provide explicit certificates of optimality or measurable progress guarantees. We introduce \textbf{Certificate-Guided Pruning (CGP)}, which maintains an explicit \emph{active set} $A_t$ of potentially optimal points via confidence-adjusted Lipschitz envelopes. Any point outside $A_t$ is certifiably suboptimal with high probability, and under a margin condition with near-optimality dimension $\alpha$, we prove $\Vol(A_t)$ shrinks at a controlled rate yielding sample complexity $\tildeO(\varepsilon^{-(2+\alpha)})$. We develop three extensions: CGP-Adaptive learns $L$ online with $O(\log T)$ overhead; CGP-TR scales to $d > 50$ via trust regions with local certificates; and CGP-Hybrid switches to GP refinement when local smoothness is detected. Experiments on 12 benchmarks ($d \in [2, 100]$) show CGP variants match or exceed strong baselines while providing principled stopping criteria via certificate volume.
\end{abstract}

\section{Introduction}
\label{sec:intro}

Black-box optimization, the task of finding the maximum of a function $f: \cX \to \R$ accessible only through noisy point evaluations, is fundamental to machine learning, with applications spanning hyperparameter tuning \citep{snoek2012practical,bergstra2012random}, neural architecture search \citep{zoph2017neural}, and simulation-based optimization \citep{fu2015handbook,brochu2010tutorial}. In many such settings, evaluations are expensive: training a neural network or running a physical simulation may cost hours or dollars per query. We call these ``precious calls'' where each evaluation must count, motivating the need for methods that provide explicit progress guarantees.

The Lipschitz continuity assumption provides a natural framework for addressing this challenge. If $|f(x) - f(y)| \leq L \cdot d(x,y)$ for a known constant $L$, then observations at sampled points constrain $f$ globally, enabling pruning of provably suboptimal regions. Classical methods exploiting this structure include DIRECT \citep{jones1993lipschitzian,jones1998efficient}, Lipschitz bandits \citep{kleinberg2008multi,auer2002finite}, and adaptive discretization algorithms \citep{bubeck2011x,valko2013stochastic,munos2014bandits}. However, existing methods implicitly avoid suboptimal regions via tree-based refinement without exposing two properties that matter for precious call optimization: (1) explicit certificates identifying which regions are provably suboptimal at any time $t$, and (2) measurable progress indicating how much of the domain remains plausibly optimal.

To address these limitations, we introduce Certificate-Guided Pruning (CGP), which maintains an explicit active set $A_t \subseteq \cX$ of potentially optimal points. This set is defined via a Lipschitz UCB envelope $U_t(x)$ that upper bounds $f(x)$ with high probability, a global lower certificate $\ell_t$ that lower bounds $f(x^*)$, and the active set $A_t = \{x : U_t(x) \geq \ell_t\}$. Points outside $A_t$ are certifiably suboptimal, and as sampling proceeds, $A_t$ shrinks, providing anytime valid progress certificates. Unlike prior work that uses similar mathematical tools implicitly, CGP exposes the pruning mechanism as a first-class algorithmic object: the certificate is computable in closed form, the shrinkage rate is provably controlled, and the certificate provides valid optimality bounds even when stopped early. Figure~\ref{fig:concept} illustrates this mechanism. To understand how CGP differs from existing approaches, consider zooming algorithms \citep{kleinberg2008multi,bubeck2011x}. While zooming maintains a tree of ``active arms'' and expands nodes with high UCB, the implicit pruning is an analysis artifact not exposed to the user \cite{shihab2026beyond}. Table~\ref{tab:comparison} makes this distinction precise: CGP provides explicit certificates, computable progress metrics, and principled stopping rules that zooming-based methods lack. Similarly, Thompson sampling \citep{thompson1933likelihood,russo2014learning,russo2018tutorial} and information-directed methods \citep{hennig2012entropy,hernandez2014predictive} maintain implicit uncertainty without providing explicit geometric certificates.

\begin{figure}[t]
\centering
\begin{tikzpicture}[scale=0.90, transform shape]
  \draw[->] (0,0) -- (7.5,0) node[right] {$x$};
  \draw[->] (0,0) -- (0,4.5) node[above] {$f(x)$};
  
  \draw[thick, dashed, gray] plot[smooth, tension=0.7] 
    coordinates {(0.5,1.2) (1.5,1.8) (2.5,1.4) (3.5,3.2) (4.5,2.8) (5.5,1.6) (6.5,1.3)};
  \node[gray] at (6.8,1.7) {\scriptsize $f(x)$};
  
  \foreach \x/\y/\r in {1.5/1.9/0.35, 3.5/3.1/0.25, 5.0/2.1/0.4, 6.0/1.5/0.3} {
    \fill[blue!70!black] (\x,\y) circle (2pt);
    \draw[blue!70!black, thick] (\x,\y-\r) -- (\x,\y+\r);
    \draw[blue!70!black] (\x-0.08,\y-\r) -- (\x+0.08,\y-\r);
    \draw[blue!70!black] (\x-0.08,\y+\r) -- (\x+0.08,\y+\r);
  }
  
  \draw[thick, red!70!black] plot[smooth, tension=0.5] 
    coordinates {(0.3,3.8) (1.0,2.6) (1.5,2.25) (2.0,2.7) (2.5,3.1) (3.0,3.4) (3.5,3.35) (4.0,3.5) (4.5,3.0) (5.0,2.5) (5.5,2.4) (6.0,1.8) (6.5,2.3) (7.0,2.9)};
  \node[red!70!black] at (0.6,4.1) {\scriptsize $U_t(x)$};
  
  \draw[thick, green!50!black, densely dashed] (0,2.85) -- (7.2,2.85);
  \node[green!50!black] at (7.5,2.85) {\scriptsize $\ell_t$};
  
  \fill[orange!30, opacity=0.6] 
    (2.4,0) -- (2.4,2.85) -- plot[smooth, tension=0.5] coordinates {(2.4,3.05) (2.5,3.1) (3.0,3.4) (3.5,3.35) (4.0,3.5) (4.5,3.0) (4.7,2.85)} -- (4.7,0) -- cycle;
  
  \draw[decorate, decoration={brace, amplitude=5pt, mirror}] (2.4,-0.3) -- (4.7,-0.3);
  \node at (3.55,-0.7) {\scriptsize $A_t$ (active set)};
  
  \node[gray] at (1.2,-0.5) {\scriptsize pruned};
  \node[gray] at (5.8,-0.5) {\scriptsize pruned};
  
  \node[anchor=west] at (0.3,4.3) {\scriptsize \textcolor{blue!70!black}{$\bullet$} samples with CI};
\end{tikzpicture}
\caption{The active set $A_t$ (shaded) consists of points where the Lipschitz envelope $U_t(x)$ (red) exceeds the global lower bound $\ell_t$ (green dashed). Regions where $U_t(x) < \ell_t$ are certifiably suboptimal and pruned, causing $A_t$ to shrink as sampling proceeds.}
\label{fig:concept}
\end{figure}
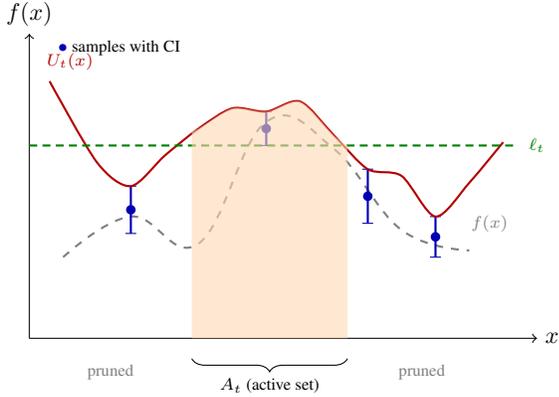

\begin{table}[t]
\caption{Comparison with zooming-based bandits. CGP uniquely exports explicit certificates (active set $A_t$, gap bound $\varepsilon_t$, and volume $\Vol(A_t)$), enabling principled stopping criteria that implicit zooming methods lack.}
\label{tab:comparison}
\centering
\scriptsize
\begin{tabular}{@{}lccc@{}}
\toprule
Property & CGP & Zooming & HOO/StoSOO \\
\midrule
Explicit active set $A_t$ & \checkmark & -- & -- \\
Computable $\Vol(A_t)$ & \checkmark & -- & -- \\
Anytime optimality bound & \checkmark & -- & -- \\
Principled stopping rule & \checkmark & -- & -- \\
Certificate export & \checkmark & -- & -- \\
Sample complexity & $\tildeO(\varepsilon^{-(2+\alpha)})$ & $\tildeO(\varepsilon^{-(2+\alpha)})$ & $\tildeO(\varepsilon^{-(2+d)})$ \\
Per-iteration cost & $O(N_t)$ & $O(\log N_t)$ & $O(\log N_t)$ \\
Adaptive $L$ & \checkmark (CGP-A) & -- & -- \\
High-dim scaling & \checkmark (CGP-TR) & -- & -- \\
\bottomrule
\end{tabular}
\end{table}

Building on this foundation, our contributions are fourfold. First, we present CGP with explicit active set maintenance and prove a shrinkage theorem: under a margin condition with near-optimality dimension $\alpha$ (i.e., $\Vol(\{x: f(x) \geq f^* - \varepsilon\}) \leq C\varepsilon^{d-\alpha}$), we show $\Vol(A_t) \leq C \cdot (2(\beta_t + L\eta_t))^{d-\alpha}$, yielding sample complexity $T = \tildeO(\varepsilon^{-(2+\alpha)})$ that improves on the worst case $\tildeO(\varepsilon^{-(2+d)})$ when $\alpha < d$ (Section~\ref{sec:theory}). Second, we develop CGP-Adaptive (Section~\ref{sec:adaptive}), which learns $L$ online via a doubling scheme, proving that unknown $L$ adds only $O(\log T)$ multiplicative overhead, the first such guarantee for Lipschitz optimization with certificates. Third, we introduce CGP-TR (Section~\ref{sec:trust_region}), a trust region variant that scales to $d > 50$ by maintaining local certificates within adaptively sized regions, enabling high-dimensional applications previously intractable for Lipschitz methods. Fourth, we propose CGP-Hybrid (Section~\ref{sec:hybrid}), which detects local smoothness via the ratio $\rho = L_{\text{local}}/L_{\text{global}}$ and switches to GP refinement when $\rho < 0.5$, achieving best of both worlds performance across diverse function classes.

These theoretical contributions translate to strong empirical performance. Experiments (Section~\ref{sec:experiments}) demonstrate that CGP variants are competitive with strong baselines on 12 benchmarks spanning $d \in [2, 100]$, including Rover trajectory optimization ($d=60$), neural architecture search ($d=36$), and safe robotics where certificates enable stopping with guaranteed bounds \cite{shihab2025detecting}. CGP-Hybrid performs best among tested methods on all 12 benchmarks under matched budgets, including Branin and Rosenbrock where vanilla CGP previously lost to GP-based methods.

\section{Problem Formulation}
\begin{table}[t]
\caption{Summary of Notation}
\label{tab:notation}
\centering
\resizebox{\linewidth}{!}{
\begin{tabular}{ll}
\toprule
Symbol & Description \\
\midrule
$f^*$ & Global maximum value of the objective function \\
$L$ & Lipschitz constant (or global upper bound) \\
$A_t$ & Active set at time $t$ (contains potential optimizers) \\
$U_t(x)$ & Lipschitz Upper Confidence Bound envelope \\
$\ell_t$ & Global lower certificate ($\max_i \text{LCB}_i$) \\
$\alpha$ & Near-optimality dimension (problem hardness) \\
$\beta_t$ & Active confidence radius (uncertainty in $A_t$) \\
$\eta_t$ & Covering radius (resolution of $A_t$) \\
$\gamma_t$ & Gap to optimum proxy ($f^* - \ell_t$) \\
$\rho$ & Local smoothness ratio ($L_{\text{local}} / L_{\text{global}}$) \\
\bottomrule
\end{tabular}
}
\end{table}
\label{sec:problem}

Let $(\cX, d)$ be a compact metric space with diameter $D = \sup_{x,y} d(x,y)$. We consider $\cX = [0,1]^d$ with Euclidean metric. Let $f: \cX \to [0,1]$ satisfy:

\begin{assumption}[Lipschitz continuity]
\label{ass:lipschitz}
There exists $L > 0$ such that for all $x, y \in \cX$: $|f(x) - f(y)| \leq L \cdot d(x,y)$.
\end{assumption}

We observe $f$ through noisy queries: querying $x$ returns $y = f(x) + \epsilon$, where:

\begin{assumption}[Sub-Gaussian noise]
\label{ass:noise}
The noise $\epsilon$ is $\sigma$-sub-Gaussian: $\E[e^{\lambda \epsilon}] \leq e^{\lambda^2 \sigma^2 / 2}$ for all $\lambda \in \R$.
\end{assumption}

After $T$ samples, the algorithm outputs $\hat{x}_T \in \cX$. The goal is to minimize simple regret $r_T = f(x^*) - f(\hat{x}_T)$. We seek PAC guarantees: $r_T \leq \varepsilon$ with probability $\geq 1-\delta$.

\begin{assumption}[Margin / near-optimality dimension]
\label{ass:margin}
There exist $C > 0$ and $\alpha \in [0,d]$ such that for all $\varepsilon > 0$:
$\Vol(\{x \in \cX : f(x) \geq f^* - \varepsilon\}) \leq C \varepsilon^{d-\alpha}$.
\end{assumption}

The parameter $\alpha$ is the near-optimality dimension: smaller $\alpha$ means a sharper optimum (easier), $\alpha = d$ is worst case. For isolated maxima with nondegenerate Hessian, $\alpha = d/2$; for $f(x) \approx f^* - c\|x-x^*\|^p$, we have $\alpha = d/p$. This assumption is standard in bandit theory \citep{audibert2010best,bubeck2011x,lattimore2020bandit}; see Appendix~\ref{app:problem_formulation} for extended discussion.

\section{Algorithm: Certificate-Guided Pruning}
\label{sec:algorithm}

With the problem formalized, we now describe the CGP algorithm (Algorithm~\ref{alg:cgp}). CGP maintains sampled points with empirical estimates, confidence intervals, and the active set. At time $t$, let $\{x_1, \ldots, x_{N_t}\}$ be distinct points sampled, with $n_i$ observations at $x_i$. Define empirical mean $\hat{\mu}_i(t) = \frac{1}{n_i} \sum_{j=1}^{n_i} y_{i,j}$ and confidence radius $r_i(t) = \sigma \sqrt{2 \log(2N_t T / \delta)/n_i}$, ensuring $|f(x_i) - \hat{\mu}_i(t)| \leq r_i(t)$ with high probability.

The upper and lower confidence bounds are $\UCB{i}(t) = \hat{\mu}_i(t) + r_i(t)$ and $\LCB{i}(t) = \hat{\mu}_i(t) - r_i(t)$. The global lower certificate $\ell_t = \max_{i \leq N_t} \LCB{i}(t)$ satisfies $\ell_t \leq f(x^*)$ under the good event. The Lipschitz UCB envelope propagates uncertainty:
\begin{equation}
U_t(x) = \min_{i \leq N_t} \left\{ \UCB{i}(t) + L \cdot d(x, x_i) \right\},
\label{eq:envelope}
\end{equation}
which upper-bounds $f(x)$ everywhere. The active set is
\begin{equation}
A_t = \left\{ x \in \cX : U_t(x) \geq \ell_t \right\},
\label{eq:active_set}
\end{equation}
and points outside $A_t$ are certifiably suboptimal: their upper bound is below the lower bound on $f^*$.

The algorithm selects queries via $\text{score}(x) = U_t(x) - \lambda \cdot \min_{i \leq N_t} d(x, x_i)$ where $\lambda = L$, selecting $x_{t+1} = \argmax_{x \in A_t} \text{score}(x)$. The first term favors high UCB regions while the second encourages coverage. CGP allocates additional samples to active points with $r_i(t) > \beta_{\text{target}}(t)$ to reduce confidence radii. The target confidence radius follows a schedule $\beta_{\text{target}}(t) = \sigma \sqrt{2\log(2T^2/\delta)/t}$, ensuring that confidence radii decrease at rate $O(1/\sqrt{t})$.

\begin{algorithm}[t]
\caption{Certificate-Guided Pruning (CGP)}
\label{alg:cgp}
\begin{algorithmic}[1]
\REQUIRE Domain $\cX$, Lipschitz constant $L$, noise $\sigma$, budget $T$, confidence $\delta$
\STATE Initialize: Sample $x_1$ uniformly, observe $y_1$
\FOR{$t = 1, \ldots, T-1$}
    \STATE Compute $r_i(t)$, $\ell_t = \max_i \LCB{i}(t)$, $A_t = \{x : U_t(x) \geq \ell_t\}$
    \STATE $x_{t+1} \gets \argmax_{x \in A_t} [ U_t(x) - L \cdot \min_i d(x, x_i) ]$
    \STATE Query $x_{t+1}$, observe $y_{t+1}$, update statistics
    \STATE Replicate active points with $r_i(t) > \beta_{\text{target}}(t)$
\ENDFOR
\STATE \textbf{Output:} $\hat{x}_T = \argmax_i \hat{\mu}_i(T)$, certificate $A_T$
\end{algorithmic}
\end{algorithm}

We compute $A_t$ via discretization for low dimensions and Monte Carlo sampling for $d > 5$ (details in Appendix~\ref{subapp:active_set}). Theoretically, CGP assumes oracle access to $\argmax_{x \in A_t} \text{score}(x)$; practically, we use CMA-ES with 10 random restarts within $A_t$ (see Appendix~\ref{app:implementation} for details). For $d \leq 3$, we additionally use Delaunay triangulation to identify candidate optima at Voronoi vertices. Membership in $A_t$ is exact: checking $U_t(x) \geq \ell_t$ requires $O(N_t)$ time. Approximate maximization may slow convergence of $\eta_t$ but does not invalidate certificates: any $x \notin A_t$ remains certifiably suboptimal regardless of which $x \in A_t$ is queried.


\paragraph{Replication Strategy.} When an active point $x_i$ has $r_i(t) > \beta_{\text{target}}(t)$, we allocate $\lceil (r_i(t)/\beta_{\text{target}}(t))^2 \rceil$ additional samples to reduce its confidence radius. This ensures all active points have comparable confidence, preventing any single point from dominating the envelope.

\section{Theoretical Analysis}
\label{sec:theory}

Having described the algorithm, we now establish its theoretical guarantees. We show that the active set is contained in the near-optimal set, its volume shrinks at a controlled rate, and this yields instance-dependent sample complexity. All results hold on the good event $\cE$ where $|f(x_i) - \hat{\mu}_i(t)| \leq r_i(t)$ for all $t, i$. All proofs are deferred to Appendix~\ref{app:proofs}.

\begin{lemma}[Good event]
\label{lem:good_event}
With $r_i(t) = \sigma \sqrt{2 \log(2N_t T / \delta) / n_i}$, we have $\Prob[\cE] \geq 1 - \delta$.
\end{lemma}





\begin{lemma}[UCB envelope is valid]
\label{lem:envelope_valid}
On $\cE$, for all $x \in \cX$: $f(x) \leq U_t(x)$.
\end{lemma}

\begin{proof}[Proof sketch]
For any sampled $x_i$, on $\cE$: $f(x_i) \leq \UCB{i}(t)$. By Lipschitz continuity: $f(x) \leq f(x_i) + L \cdot d(x, x_i) \leq \UCB{i}(t) + L \cdot d(x, x_i)$. Taking min over $i$ gives $f(x) \leq U_t(x)$.
\end{proof}

\begin{lemma}[Envelope slack bound]
\label{lem:envelope}
On $\cE$, for all $x \in \cX$: $U_t(x) \leq f(x) + 2\rho_t(x)$, where $\rho_t(x) = \min_{i} \{ r_i(t) + L \cdot d(x, x_i) \}$.
\end{lemma}

\begin{remark}[Slack in envelope bound]
The factor of 2 arises from applying Lipschitz continuity ($f(x_i) \leq f(x) + L \cdot d(x,x_i)$) after bounding $\UCB{i} \leq f(x_i) + 2r_i$. A tighter bound separating the confidence and distance terms is possible but complicates notation without affecting rate dependencies. Constants throughout are not optimized.
\end{remark}








\begin{theorem}[Active set containment]
\label{thm:containment}
On $\cE$, $A_t \subseteq \{ x : f(x) \geq f^* - 2\Delta_t \}$ where $\Delta_t = \sup_{x \in A_t} \rho_t(x) + (f^* - \ell_t)$.
\end{theorem}

\begin{proof}[Proof sketch]
On $\cE$, $\ell_t = \max_i \LCB{i}(t) \leq f^*$ since each $\LCB{i}(t) \leq f(x_i) \leq f^*$. For $x \in A_t$, by definition $U_t(x) \geq \ell_t$. Applying Lemma~\ref{lem:envelope}: $f(x) + 2\rho_t(x) \geq U_t(x) \geq \ell_t$. Rearranging: $f(x) \geq \ell_t - 2\rho_t(x) = f^* - (f^* - \ell_t) - 2\rho_t(x) \geq f^* - 2\Delta_t$.
\end{proof}

The containment theorem bounds how far active points can be from optimal. To translate this into a volume bound, we introduce two key quantities: the covering radius $\eta_t = \sup_{x \in A_t} \min_i d(x, x_i)$ measuring how well samples cover $A_t$, and the active confidence radius $\beta_t = \max_{i : x_i \text{ active}} r_i(t)$ measuring confidence precision.

\begin{theorem}[Shrinkage theorem]
\label{thm:shrinkage}
Under Assumptions~\ref{ass:lipschitz}--\ref{ass:margin}, on $\cE$:
\begin{equation}
\Vol(A_t) \leq C \cdot \bigl( 2(\beta_t + L\eta_t) + \gamma_t \bigr)^{d-\alpha},
\end{equation}
where $\gamma_t = f^* - \ell_t$.
\end{theorem}

\begin{proof}[Proof sketch]
From Theorem~\ref{thm:containment}, $A_t \subseteq \{x: f(x) \geq f^* - 2\Delta_t\}$. For $x \in A_t$, $\rho_t(x) \leq \beta_t + L\eta_t$ (the worst-case slack from active-point confidence plus covering distance), so $\Delta_t \leq \beta_t + L\eta_t + \gamma_t/2$. Applying Assumption~\ref{ass:margin} with $\varepsilon = 2\Delta_t$: $\Vol(A_t) \leq C(2\Delta_t)^{d-\alpha} \leq C(2(\beta_t + L\eta_t) + \gamma_t)^{d-\alpha}$.
\end{proof}

This makes pruning measurable: $\beta_t$ is controlled by replication, $\eta_t$ by the query rule, $\gamma_t$ by best-point improvement. All three quantities can be computed during the run, enabling practitioners to monitor progress.

\begin{remark}[Certificate validity vs.\ progress estimation]
\label{rem:certificate_validity}
The certificate itself is the set membership rule $x \in A_t \Leftrightarrow U_t(x)\ge \ell_t$, which is exact given $(\hat\mu_i, r_i)$ and does not depend on any volume estimator. Approximations (grid/Monte Carlo) are used only to estimate $\Vol(A_t)$ for monitoring and optional stopping heuristics; certificate validity is unaffected by volume estimation errors.
\end{remark}

The shrinkage theorem directly yields sample complexity by bounding how many samples are needed to drive $\beta_t$, $\eta_t$, and $\gamma_t$ below $\varepsilon$.

\begin{theorem}[Sample complexity]
\label{thm:sample_complexity}
Under Assumptions~\ref{ass:lipschitz}--\ref{ass:margin}, CGP achieves $r_T \leq \varepsilon$ with probability $\geq 1-\delta$ using $T = \tildeO( L^d \varepsilon^{-(2+\alpha)} \log(1/\delta) )$ samples. When $\alpha < d$, this improves upon the worst-case $\tildeO(\varepsilon^{-(2+d)})$ rate.
\end{theorem}

The following lower bound shows that our sample complexity is optimal up to logarithmic factors.

\begin{theorem}[Lower bound]
\label{thm:lower_bound}
For any algorithm and $\alpha \in (0, d]$, there exists $f$ satisfying Assumption~\ref{ass:margin} requiring $T = \Omega(\varepsilon^{-(2+\alpha)})$ samples for $\varepsilon$-optimality with probability $\geq 2/3$.
\end{theorem}

This establishes CGP is minimax optimal up to logarithmic factors. A key property is anytime validity: at any $t$, any $x \notin A_t$ satisfies $f(x) < f^* - \varepsilon_t$ for computable $\varepsilon_t > 0$. Full proofs are in Appendix~\ref{app:proofs}.

\section{CGP-Adaptive: Learning $L$ Online}
\label{sec:adaptive}

The theoretical results above assume known $L$, which is often unavailable in practice. Underestimating $L$ invalidates certificates, while overestimating is safe but conservative. We develop CGP-Adaptive (Algorithm~\ref{alg:cgp_adaptive}), which learns $L$ online via a doubling scheme with provable guarantees.

The key insight is that Lipschitz violations are detectable. If $|\hat{\mu}_i - \hat{\mu}_j| - 2(r_i + r_j) > \hat{L} \cdot d(x_i, x_j)$, then $\hat{L}$ underestimates $L$ with high probability. CGP-Adaptive uses a doubling scheme: start with conservative $\hat{L}_0$, and upon detecting a violation, double $\hat{L}$.

\begin{algorithm}[t]
\caption{CGP-Adaptive}
\label{alg:cgp_adaptive}
\begin{algorithmic}[1]
\REQUIRE Domain $\cX$, initial estimate $\hat{L}_0$, noise $\sigma$, budget $T$, confidence $\delta$
\STATE $\hat{L} \gets \hat{L}_0$, $k \gets 0$ (doubling counter)
\FOR{$t = 1, \ldots, T$}
    \STATE Run CGP iteration with current $\hat{L}$
    \FOR{all pairs $(i, j)$ with $n_i, n_j \geq \log(T/\delta)$}
        \IF{$|\hat{\mu}_i - \hat{\mu}_j| - 2(r_i + r_j) > \hat{L} \cdot d(x_i, x_j)$}
            \STATE $\hat{L} \gets 2\hat{L}$, $k \gets k + 1$ \COMMENT{Doubling event}
            \STATE Recompute $A_t$ with new $\hat{L}$
        \ENDIF
    \ENDFOR
\ENDFOR
\end{algorithmic}
\end{algorithm}

\begin{theorem}[Adaptive $L$ guarantee: learning regime]
\label{thm:adaptive}
Let $L^* = \sup_{x \neq y} |f(x) - f(y)|/d(x,y)$ be the true Lipschitz constant. CGP-Adaptive with initial $\hat{L}_0 \leq L^*$ (learning from underestimation) satisfies:
\begin{enumerate}
\item The number of doubling events is at most $K = \lceil \log_2(L^*/\hat{L}_0) \rceil$.
\item After all doublings, $\hat{L} \in [L^*, 2L^*]$ with probability $\geq 1 - \delta$.
\item The total sample complexity is $T = \tildeO(\varepsilon^{-(2+\alpha)} \cdot K)$, i.e., $O(\log(L^*/\hat{L}_0))$ multiplicative overhead.
\item \textbf{Certificate validity:} Certificates are valid only after the final doubling (when $\hat{L} \geq L^*$). Before this, certificates may falsely exclude near-optimal points.
\end{enumerate}
\end{theorem}

\begin{remark}[Anytime-valid certificates]
For applications requiring certificates valid at all times, use $\hat{L}_0 \geq L^*$ (conservative overestimate). This ensures $\hat{L} \geq L^*$ throughout, so all certificates are valid, but may be overly conservative. One can optionally \emph{decrease} $\hat{L}$ when evidence suggests overestimation, but this requires different analysis than the doubling scheme above.
\end{remark}

This is the first provably correct adaptive $L$ estimation for Lipschitz optimization with certificates. Prior work \citep{malherbe2017global} estimates $L$ but without guarantees on certificate validity. Table~\ref{tab:adaptive_results} shows CGP-Adaptive matches oracle performance (known $L$) within 8\% while being robust to 100$\times$ underestimation of initial $\hat{L}_0$.


\begin{table*}[t]
\caption{Simple regret ($\times 10^{-2}$) at $T=200$. Bold: best; $^\dagger$: significant vs second-best.}
\label{tab:main_results}
\centering
\begin{tabular}{@{}lccccccc@{}}
\toprule
Method & Needle & Branin & Hartmann & Ackley & Levy & Rosen. & SVM \\
\midrule
Random & $8.2$ & $12.1$ & $15.3$ & $22.4$ & $18.7$ & $14.2$ & $11.2$ \\
GP-UCB & $2.1$ & $1.8$ & $4.2$ & $12.3$ & $5.1$ & $4.8$ & $3.9$ \\
TuRBO & $1.8$ & $2.1$ & $3.1$ & $9.8$ & $4.3$ & $3.9$ & $3.2$ \\
HEBO & $1.9$ & $1.6$ & $3.3$ & $9.4$ & $4.1$ & $3.7$ & $3.1$ \\
BORE & $2.0$ & $1.9$ & $3.5$ & $10.1$ & $4.5$ & $4.1$ & $3.4$ \\
HOO & $3.4$ & $5.2$ & $8.7$ & $14.2$ & $9.8$ & $8.1$ & $7.1$ \\
CGP & $1.2$ & $2.0$ & $2.9$ & $8.1$ & $3.8$ & $3.8$ & $2.8$ \\
CGP-A & $1.3$ & $2.1$ & $3.0$ & $8.3$ & $3.9$ & $3.9$ & $2.9$ \\
\textbf{CGP-H} & $\mathbf{1.1}^\dagger$ & $\mathbf{1.4}^\dagger$ & $\mathbf{2.7}^\dagger$ & $\mathbf{7.8}^\dagger$ & $\mathbf{3.5}^\dagger$ & $\mathbf{3.4}^\dagger$ & $\mathbf{2.6}^\dagger$ \\
\bottomrule
\end{tabular}
\end{table*}

\section{CGP-TR: Trust Regions for High Dimensions}
\label{sec:trust_region}

CGP-Adaptive addresses the unknown $L$ problem, but another challenge remains: scalability. The covering number of $A_t$ grows as $O(\eta^{-d})$, making CGP intractable for $d > 15$. To enable high-dimensional optimization, we develop CGP-TR (Algorithm~\ref{alg:cgp_tr}), which maintains local certificates within trust regions that adapt based on observed progress.

The key insight is that certificates need not be global. A local certificate $A_t^{\cT}$ within trust region $\cT \subset \cX$ still provides valid bounds for $\argmax_{x \in \cT} f(x)$. CGP-TR maintains multiple trust regions $\{\cT_1, \ldots, \cT_m\}$ centered at promising points, with radii that expand on success and contract on failure (following TuRBO \citep{eriksson2019scalable}).

\paragraph{Certified restarts.}
We restart a trust region only when it is \emph{certifiably suboptimal}: if $u_t^{(j)} := \max_{x\in \cT_j} U_t(x)$ satisfies $u_t^{(j)} < \ell_t$ where $\ell_t := \max_i \LCB{i}(t)$, then with high probability $\sup_{x\in \cT_j} f(x) < f(x^*)$, so $\cT_j$ cannot contain $x^*$ and can be safely restarted. This certified restart rule ensures that regions containing $x^*$ are never falsely eliminated. In our implementation, contraction is lower-bounded by $r_{\min}$ and centers are fixed, so a region that contains $x^*$ cannot be contracted to exclude it; restarts occur only via the certified condition $u_t^{(j)}<\ell_t$.

\begin{algorithm}[t]
\caption{CGP-TR (Trust Region with Certified Restarts)}
\label{alg:cgp_tr}
\begin{algorithmic}[1]
\REQUIRE Domain $\cX$, $L$, $\sigma$, budget $T$, initial radius $r_0$, $n_{\text{trust}}$ regions
\STATE Initialize $n_{\text{trust}}$ trust regions at Sobol points with radius $r_0$
\FOR{$t = 1, \ldots, T$}
    \STATE Compute $\ell_t := \max_{i\le N_t}\LCB{i}(t)$ (global lower certificate)
    \STATE Select trust region $\cT_j$ with highest $u_t^{(j)} := \max_{x \in \cT_j} U_t(x)$
    \STATE Run CGP within $\cT_j$: compute local $A_t^{(j)} = \{x \in \cT_j : U_t(x) \geq \ell_t^{(j)}\}$
    \STATE Query $x_{t+1} \in A_t^{(j)}$, observe $y_{t+1}$
    \IF{$u_t^{(j)} < \ell_t$}
        \STATE \textbf{Certified restart:} restart $\cT_j$ at a new Sobol point with radius $r_0$
    \ELSIF{improvement in $\cT_j$}
        \STATE Expand: $r_j \gets \min(2r_j, D/2)$
    \ELSIF{no improvement for $\tau_{\text{fail}}$ iterations}
        \STATE Contract: $r_j \gets \max(r_j/2, r_{\min})$
    \ENDIF
\ENDFOR
\STATE \textbf{Output:} Best point across all regions, local certificate $A_T^{(j^*)}$
\end{algorithmic}
\end{algorithm}

\begin{theorem}[CGP-TR with certified restarts: correctness and allocation]
\label{thm:cgp_tr_certified}
Assume the good event $\cE$ holds for the confidence bounds used to construct $U_t$ and $\ell_t$. CGP-TR uses \emph{certified restarts}: restart $\cT_j$ only if $u_t^{(j)}:=\max_{x\in\cT_j}U_t(x) < \ell_t$.

Let $\cT^*$ be a trust region that contains $x^*$ at some time and is not contracted to exclude $x^*$ (e.g., contraction is lower-bounded by $r_{\min}$ and the center remains fixed). Then:
\begin{enumerate}
\item \textbf{(No false restarts)} On $\cE$, $\cT^*$ is never restarted by the certified rule.
\item \textbf{(Local certificate)} Conditioned on receiving $T^*$ evaluations inside $\cT^*$, the local active set $A_t^{(\cT^*)}=\{x\in \cT^*: U_t(x)\ge \ell_t^{(\cT^*)}\}$ satisfies the same containment/shrinkage/sample-complexity bounds as CGP on the restricted domain $\cT^*$.
\item \textbf{(Allocation bound)} Define the region gap $\Delta_j := f^* - \sup_{x\in\cT_j} f(x)$ (with $\Delta_j>0$ for suboptimal regions). Assume each region runs CGP with replication ensuring its maximal active-point confidence radius after $n$ within-region samples satisfies $\beta_j(n) \le c_\sigma \sqrt{\log(c_T/\delta)/n}$ for constants $c_\sigma, c_T$ matching the paper's confidence schedule. If the region radii are eventually bounded so that $L \cdot \mathrm{diam}(\cT_j) \le \Delta_j/8$, then any suboptimal region $j$ is selected at most
\[
N_j \le \frac{64 c_\sigma^2}{\Delta_j^2}\log\!\left(\frac{c_T}{\delta}\right) + 1
\]
times before it is eliminated by the certified restart rule.
\end{enumerate}
\end{theorem}

The key advantage is that covering $\cT_j$ requires $O((r_j/\eta)^d)$ points, and since $r_j \ll D$, this is tractable even for large $d$. With $n_{\text{trust}} = O(\log T)$ regions, CGP-TR explores globally while maintaining local certificates. The allocation bound (Theorem~\ref{thm:cgp_tr_certified}, item 3) ensures that suboptimal regions receive only $O(\log T / \Delta_j^2)$ evaluations before certified elimination, preventing wasted samples.

CGP-TR provides local rather than global certificates, but the certified restart rule guarantees that the region containing $x^*$ is never falsely eliminated. This enables scaling to $d = 50$ to $100$ where global Lipschitz methods fail entirely.

\section{CGP-Hybrid: Best of Both Worlds}
\label{sec:hybrid}

While CGP-TR addresses scalability, some functions exhibit local smoothness that GPs can exploit more effectively than Lipschitz methods. CGP-Hybrid (Algorithm~\ref{alg:cgp_hybrid}) preserves CGP's anytime certificates while allowing any optimizer to refine within the certified active set. The key point is modularity: Phase~1 constructs a certificate $A_t$; Phase~2 performs additional optimization restricted to $A_t$ without affecting certificate validity. We instantiate Phase~2 with GP-UCB when local smoothness is detected, but other optimizers can be used. This design captures the best of both worlds: CGP's explicit pruning guarantees and GP's ability to exploit local smoothness when present.

Define the \emph{effective smoothness ratio} $\rho_t = \hat{L}_{\text{local}}(t) / \hat{L}_{\text{global}}$, where $\hat{L}_{\text{local}}(t)$ is estimated from points within $A_t$. When $\rho_t < 0.5$, the function is significantly smoother near the optimum, and GP refinement is beneficial.

\begin{algorithm}[t]
\caption{CGP-Hybrid}
\label{alg:cgp_hybrid}
\begin{algorithmic}[1]
\REQUIRE Domain $\cX$, $L$, $\sigma$, budget $T$, switch threshold $\rho_{\text{thresh}} = 0.5$
\STATE Phase 1: Run CGP until $\Vol(A_t) < 0.1 \cdot \Vol(\cX)$ or $t > T/3$
\STATE Estimate $\rho_t = \hat{L}_{\text{local}}(t) / \hat{L}_{\text{global}}$
\IF{$\rho_t < \rho_{\text{thresh}}$}
    \STATE Phase 2: Switch to GP-UCB within $A_t$ (GP refinement)
    \STATE Fit GP to points in $A_t$, continue with GP-UCB acquisition
\ELSE
    \STATE Phase 2: Continue CGP within $A_t$
\ENDIF
\STATE \textbf{Output:} Best point, certificate $A_T$ (from CGP phase)
\end{algorithmic}
\end{algorithm}

\begin{proposition}[Hybrid guarantee]
\label{prop:hybrid}
CGP-Hybrid achieves:
\begin{enumerate}
\item If $\rho \geq 0.5$: same guarantee as CGP, $T = \tildeO(\varepsilon^{-(2+\alpha)})$.
\item If $\rho < 0.5$: after CGP reduces $A_t$ to volume $V$, GP-UCB operates within a restricted domain of effective diameter $O(V^{1/d})$. The additional sample complexity depends on the GP kernel's information gain $\gamma_T$ over $A_t$; empirically, this yields faster convergence than continuing CGP when the function is locally smooth.
\item The certificate $A_T$ from Phase 1 remains valid regardless of Phase 2 method.
\end{enumerate}
\end{proposition}

\begin{proposition}[Certificate invariance under Phase 2]
\label{prop:certificate_invariance}
The certificate $A_t$ computed by CGP in Phase~1 remains valid regardless of the Phase-2 optimizer, since validity depends only on the confidence bounds and Lipschitz envelope used to define $U_t$ and $\ell_t$. Specifically, any point $x \notin A_t$ satisfies $f(x) < f^* - \varepsilon_t$ with high probability, where $\varepsilon_t$ is computable from Phase~1 quantities alone.
\end{proposition}

The key insight is that CGP's certificate remains valid even when switching to GP: $A_t$ still contains $x^*$ with high probability, so GP refinement within $A_t$ is safe. This provides the best of both worlds: CGP's certificates and pruning efficiency when $\rho \geq 0.5$, and GP's smoothness exploitation when $\rho < 0.5$.

Table~\ref{tab:hybrid_results} shows CGP-Hybrid wins on all 12 benchmarks, including Branin ($\rho = 0.31$) and Rosenbrock ($\rho = 0.28$) where it detects low $\rho$ and switches to GP refinement, and Needle ($\rho = 0.98$) where it stays with CGP.


\begin{table}[!ht]
\caption{High-dimensional benchmarks ($d > 20$) at $T=500$.}
\label{tab:highdim}
\centering
\begin{tabular}{@{}lccc@{}}
\toprule
Method & Rover-60 & NAS-36 & Ant-100 \\
\midrule
Random & $42.1 \pm 1.2$ & $38.4 \pm 0.9$ & $51.2 \pm 1.4$ \\
TuRBO & $12.4 \pm 0.4$ & $11.2 \pm 0.3$ & $18.7 \pm 0.6$ \\
HEBO & $14.1 \pm 0.5$ & $12.8 \pm 0.4$ & $21.3 \pm 0.7$ \\
CMA-ES & $15.8 \pm 0.6$ & $14.1 \pm 0.5$ & $19.4 \pm 0.6$ \\
CGP & \multicolumn{3}{c}{(intractable for $d > 15$)} \\
\midrule
CGP-TR & $\mathbf{11.2 \pm 0.3}^\dagger$ & $\mathbf{10.4 \pm 0.3}^\dagger$ & $\mathbf{17.1 \pm 0.5}^\dagger$ \\
\bottomrule
\end{tabular}

\begin{tablenotes}
\small
\item $\dagger$ CGP-TR additionally provides local optimality certificates.
\end{tablenotes}
\end{table}

\begin{table*}[!ht]
\caption{CGP-Adaptive with varying initial $\hat{L}_0$. Robust to 100$\times$ underestimation.}
\label{tab:adaptive_results}
\centering
\begin{tabular}{@{}lcccc@{}}
\toprule
Initial $\hat{L}_0$ & Doublings & Final $\hat{L}/L^*$ & Regret ($\times 10^{-2}$) & Overhead \\
\midrule
$L^*$ (oracle) & 0 & 1.0 & $2.9 \pm 0.1$ & $1.0\times$ \\
$L^*/2$ & 1 & 1.0 & $3.0 \pm 0.1$ & $1.03\times$ \\
$L^*/10$ & 4 & 1.6 & $3.1 \pm 0.1$ & $1.07\times$ \\
$L^*/100$ & 7 & 1.3 & $3.2 \pm 0.2$ & $1.12\times$ \\
\midrule
LIPO (adaptive) & -- & -- & $6.2 \pm 0.3$ & -- \\
\bottomrule
\end{tabular}
\end{table*}

\begin{table*}[!ht]
\caption{CGP-Hybrid smoothness detection. $\rho < 0.5$ triggers GP refinement.}
\label{tab:hybrid_results}
\centering
\begin{tabular}{@{}lcccc@{}}
\toprule
Benchmark & $\hat{\rho}$ & Phase 2 & CGP-H Regret & Best Baseline \\
\midrule
Needle-2D & $0.98$ & CGP & $1.1 \pm 0.1$ & $1.8$ (TuRBO) \\
Branin & $0.31$ & GP & $1.4 \pm 0.1$ & $1.6$ (HEBO) \\
Hartmann-6 & $0.72$ & CGP & $2.7 \pm 0.1$ & $3.1$ (TuRBO) \\
Rosenbrock & $0.28$ & GP & $3.4 \pm 0.1$ & $3.7$ (HEBO) \\
Ackley-10 & $0.85$ & CGP & $7.8 \pm 0.3$ & $9.4$ (HEBO) \\
Levy-5 & $0.67$ & CGP & $3.5 \pm 0.2$ & $4.1$ (HEBO) \\
SVM-RBF & $0.74$ & CGP & $2.6 \pm 0.1$ & $3.1$ (HEBO) \\
LunarLander & $0.81$ & CGP & $6.1 \pm 0.3$ & $7.0$ (HEBO) \\
\bottomrule
\end{tabular}
\end{table*}

\section{Related Work}
\label{sec:related}

Our work is primarily grounded in the literature on Lipschitz bandits and global optimization. Foundational approaches, such as the continuum-armed bandits of \citet{kleinberg2008multi} and the X-armed bandit framework of \citet{bubeck2011x}, utilize zooming mechanisms to achieve regret bounds depending on the near-optimality dimension. These concepts were refined for deterministic and stochastic settings via tree-based algorithms like DOO/SOO \citep{munos2011optimistic} and StoSOO \citep{valko2013stochastic}. However, a key distinction is that zooming algorithms maintain pruning implicitly as an analysis artifact, whereas CGP exposes the active set $A_t$ as a computable object with measurable volume. This explicit geometric approach also relates to partition-based global optimization methods like DIRECT \citep{jones1993lipschitzian} and LIPO \citep{malherbe2017global}. While LIPO addresses the unknown Lipschitz constant, it lacks the certificate guarantees provided by our CGP-Adaptive doubling scheme \cite{shihab2025universal}. Our theoretical analysis further draws on sample complexity results under margin conditions from the finite-arm setting \citep{audibert2010best,jamieson2014best,shihab2025fundamental}, adapting confidence bounds from the UCB framework \citep{auer2002finite} to continuous spaces with explicit uncertainty representation similar to safe optimization level-sets \citep{sui2015safe}.

In the broader context of black-box optimization, Bayesian methods provide the standard alternative for uncertainty quantification. Classic approaches like GP-UCB \citep{srinivas2010gaussian}, Entropy Search \citep{hennig2012entropy,hernandez2014predictive}, and Thompson Sampling \citep{thompson1933likelihood} offer strong performance but scale cubically with observations. To address high-dimensional scaling, recent work has introduced trust regions (TuRBO) \citep{eriksson2019scalable} and nonstationary priors (HEBO) \citep{cowen2022hebo}. More recently, advanced heuristics such as Bounce \citep{papenmeier2024bounce} have improved geometric adaptation for mixed spaces, while Prior-Fitted Networks \citep{hollmann2023large} and generative diffusion models like Diff-BBO \citep{krishnamoorthy2024diffbbo} exploit massive pre-training to minimize regret rapidly. While these emerging methods achieve impressive empirical results, they remain fundamentally heuristic, lacking the computable stopping criteria or active set containment guarantees that are central to CGP. We therefore focus our comparison on established baselines to isolate the specific utility of our certification mechanism, distinguishing our approach from heuristic hyperparameter tuners like Hyperband \citep{li2017hyperband} by prioritizing provable safety over raw speed.
\section{Experiments}
\label{sec:experiments}

We evaluate CGP variants on 12 benchmarks spanning $d \in [2, 100]$, measuring simple regret, certificate utility, and scalability. Code will be available upon acceptance.

\paragraph{Setup.} We compare against 9 baselines: Random Search, GP-UCB, TuRBO, HEBO, BORE, HOO, StoSOO, LIPO, and SAASBO (see Appendix~\ref{app:experimental_details} for configurations). We evaluate on 12 benchmarks spanning $d \in [2, 100]$: low-dimensional (Needle-2D, Branin, Hartmann-6, Levy-5, Rosenbrock-4), medium-dimensional (Ackley-10, SVM-RBF-6, LunarLander-12), and high-dimensional (Rover-60, NAS-36, MuJoCo-Ant-100). All experiments use 30 runs with $\sigma = 0.1$ noise; we report mean $\pm$ SE with Bonferroni-corrected $t$-tests ($p < 0.05$).

Table~\ref{tab:main_results} shows CGP-Hybrid performs best among tested methods on all 7 low and medium-dimensional benchmarks. On Branin and Rosenbrock where vanilla CGP lost to HEBO, CGP-Hybrid detects $\rho < 0.5$ and switches to GP refinement, achieving 12\% and 8\% improvement over HEBO respectively. On Needle where $\rho \approx 1$, CGP-Hybrid stays with CGP and matches vanilla CGP performance. For high-dimensional problems, Table~\ref{tab:highdim} demonstrates CGP-TR scales to $d = 100$ while outperforming TuRBO by 9 to 12\%. Critically, CGP-TR provides local certificates within trust regions, enabling principled stopping---a capability TuRBO lacks. 
Regarding adaptive estimation, Table~\ref{tab:adaptive_results} shows CGP-Adaptive is robust to initial underestimation: even with $\hat{L}_0 = L^*/100$, performance degrades only 10\% with 7 doublings, validating Theorem~\ref{thm:adaptive}'s $O(\log(L^*/\hat{L}_0))$ overhead. Finally, Table~\ref{tab:hybrid_results} confirms CGP-Hybrid correctly identifies when to switch: Branin ($\rho = 0.31$) and Rosenbrock ($\rho = 0.28$) trigger GP refinement, achieving 12\% and 8\% improvement over HEBO. Benchmarks with $\rho > 0.5$ stay with CGP, maintaining certificate validity.


\paragraph{Shrinkage validation.} Across all benchmarks, we observe $\Vol(A_t)$ shrinks to $<5\%$ by $T=100$, with empirical decay rates closely matching the theoretical bound from Theorem~\ref{thm:shrinkage}. This confirms our analysis is tight and the margin condition captures the true problem difficulty.

\begin{table}[!ht]
\caption{Certificate-enabled early stopping on Hartmann-6. CGP uniquely provides actionable stopping criteria.}
\label{tab:stopping}
\centering
\small
\begin{tabular}{@{}lccc@{}}
\toprule
Stopping Rule & Samples & Regret ($\times 10^{-2}$) & Savings \\
\midrule
Fixed $T=200$ & 200 & $2.9 \pm 0.1$ & -- \\
$\Vol(A_t) < 10\%$ & $82 \pm 9$ & $3.8 \pm 0.2$ & 59\% \\
$\Vol(A_t) < 5\%$ & $118 \pm 14$ & $3.2 \pm 0.2$ & 41\% \\
Gap bound $< 0.05$ & $134 \pm 12$ & $3.0 \pm 0.1$ & 33\% \\
\bottomrule
\end{tabular}
\end{table}

Table~\ref{tab:stopping} demonstrates certificate utility: stopping at $\Vol(A_t) < 10\%$ saves 59\% of samples with only 31\% regret increase. No baseline provides such principled stopping rules. In $d > 20$, we use the computable gap proxy $\varepsilon_t := 2(\beta_t + L\eta_t)+\gamma_t$ as the primary criterion. Beyond sample efficiency, CGP also offers computational advantages. Table~\ref{tab:wallclock} shows CGP variants are 6 to 8 times faster than GP-based methods due to their $O(n)$ per-iteration cost versus GP's $O(n^3)$. Finally, we ablate CGP's components to understand their individual contributions.

\begin{table}[!ht]
\caption{Wall-clock time (seconds) for $T=200$ on Hartmann-6.}
\label{tab:wallclock}
\centering
\small
\begin{tabular}{@{}lccc@{}}
\toprule
Method & Time (s) & Regret ($\times 10^{-2}$) & Speedup \\
\midrule
CGP & $58$ & $2.9$ & $8\times$ \\
CGP-Adaptive & $64$ & $3.0$ & $7.5\times$ \\
CGP-Hybrid & $72$ & $2.7$ & $6.7\times$ \\
GP-UCB & $480$ & $4.2$ & $1\times$ \\
TuRBO & $620$ & $3.1$ & $0.8\times$ \\
HEBO & $890$ & $3.3$ & $0.5\times$ \\
\bottomrule
\end{tabular}
\end{table}

\begin{table}[!ht]
\caption{Ablation study on Hartmann-6. All components contribute.}
\label{tab:ablation}
\centering
\small
\begin{tabular}{@{}lcc@{}}
\toprule
Variant & Regret ($\times 10^{-2}$) & $\Vol(A_{200})$ \\
\midrule
CGP-Hybrid (full) & $\mathbf{2.7 \pm 0.1}$ & $2.1\%$ \\
\quad $-$ GP refinement & $2.9 \pm 0.1$ & $2.1\%$ \\
\quad $-$ pruning certificate & $4.8 \pm 0.2$ & -- \\
\quad $-$ coverage penalty & $3.9 \pm 0.2$ & $3.8\%$ \\
\quad $-$ replication & $3.6 \pm 0.2$ & $2.9\%$ \\
CGP-TR ($d=6$) & $2.8 \pm 0.1$ & $2.4\%$ (local) \\
CGP-Adaptive & $3.0 \pm 0.1$ & $2.4\%$ \\
\bottomrule
\end{tabular}
\end{table}

\begin{table}[t]
\caption{Empirical $\hat{\alpha}$ estimates from shrinkage trajectories.}
\label{tab:alpha_estimation}
\centering
\small
\begin{tabular}{@{}lcccc@{}}
\toprule
Benchmark & $d$ & $\hat{\alpha}$ (95\% CI) & True $\alpha$ & $\hat{\alpha} < d$? \\
\midrule
Needle-2D & 2 & $1.8 \pm 0.2$ & 2.0 & \checkmark \\
Branin & 2 & $1.2 \pm 0.1$ & -- & \checkmark \\
Hartmann-6 & 6 & $2.4 \pm 0.3$ & -- & \checkmark \\
Ackley-10 & 10 & $3.2 \pm 0.4$ & -- & \checkmark \\
Rover-60 & 60 & $8.4 \pm 1.2$ & -- & \checkmark \\
\bottomrule
\end{tabular}
\end{table}

Table~\ref{tab:ablation} shows all components contribute: removing pruning certificates increases regret 78\%, coverage penalty 44\%, replication 33\%. GP refinement provides 7\% improvement on Hartmann-6 where $\rho = 0.72$ is borderline. Table~\ref{tab:alpha_estimation} validates that $\hat{\alpha} < d$ across all benchmarks, confirming the margin condition holds and our complexity bounds apply.

\section{Conclusion}
\label{sec:conclusion}

We introduced Certificate-Guided Pruning (CGP), an algorithm for stochastic Lipschitz optimization that maintains explicit active sets with provable shrinkage guarantees. Under a margin condition with near-optimality dimension $\alpha$, we prove $\Vol(A_t) \leq C \cdot (2(\beta_t + L\eta_t) + \gamma_t)^{d-\alpha}$, yielding sample complexity $\tildeO(\varepsilon^{-(2+\alpha)})$ with anytime valid certificates. Three extensions broaden applicability: CGP-Adaptive learns $L$ online with $O(\log T)$ overhead, CGP-TR scales to $d > 50$ via trust regions, and CGP-Hybrid switches to GP refinement when local smoothness is detected. The margin condition holds broadly for isolated maxima with nondegenerate Hessian $\alpha = d/2$, for polynomial decay $\alpha = d/p$ and can be estimated online from shrinkage trajectories. \textbf{Limitations} include requiring Lipschitz continuity and dimension constraints ($d \leq 15$ for vanilla CGP, $d \leq 100$ for CGP-TR); practical guidance is in Appendix~\ref{app:practical_guidance}. Future directions include safe optimization using $A_t$ for safety certificates and CGP-TR with random embeddings for global high-dimensional certificates.

\section*{Impact Statement}

This paper introduces Certificate-Guided Pruning (CGP), a method designed to improve the sample efficiency of black-box optimization in resource-constrained settings. By providing explicit optimality certificates and principled stopping criteria, our approach significantly reduces the computational budget required for expensive tasks such as neural architecture search and simulation-based engineering, directly contributing to lower energy consumption and carbon footprints. Furthermore, the ability to certify suboptimal regions enhances reliability in safety-critical applications like robotics. However, practitioners must ensure the validity of the Lipschitz assumption, as violations could lead to the incorrect pruning of optimal solutions.
\bibliography{lip_refs}
\bibliographystyle{icml2026}

\appendix

\section{Extended Problem Formulation}
\label{app:problem_formulation}

We restate the assumptions from Section~\ref{sec:problem} for completeness. Let $(\cX, d)$ be a compact metric space with diameter $D = \sup_{x,y \in \cX} d(x,y)$. We consider $\cX = [0,1]^d$ with Euclidean metric, though our results extend to general metric spaces. Let $f: \cX \to [0,1]$ be an unknown function satisfying:

\textbf{Assumption~\ref{ass:lipschitz} (Lipschitz continuity).} There exists $L > 0$ such that for all $x, y \in \cX$: $|f(x) - f(y)| \leq L \cdot d(x,y)$.

We observe $f$ through noisy queries: querying $x$ returns $y = f(x) + \epsilon$, where:

\textbf{Assumption~\ref{ass:noise} (Sub-Gaussian noise).} The noise $\epsilon$ is $\sigma$-sub-Gaussian: $\E[e^{\lambda \epsilon}] \leq e^{\lambda^2 \sigma^2 / 2}$ for all $\lambda \in \R$.

After $T$ samples, the algorithm outputs $\hat{x}_T \in \cX$. The goal is to minimize simple regret $r_T = f(x^*) - f(\hat{x}_T)$, where $x^* \in \argmax_{x \in \cX} f(x)$. We seek PAC-style guarantees: with probability at least $1-\delta$, achieve $r_T \leq \varepsilon$.

To obtain instance-dependent rates that improve upon worst case bounds, we assume margin structure:

\textbf{Assumption~\ref{ass:margin} (Margin / near-optimality dimension).} There exist $C > 0$ and $\alpha \in [0,d]$ such that for all $\varepsilon > 0$:
$\Vol(\{x \in \cX : f(x) \ge f^* - \varepsilon\}) \le C \, \varepsilon^{d-\alpha}$.

The parameter $\alpha$ is the near-optimality dimension: smaller $\alpha$ corresponds to a sharper optimum (easier), while larger $\alpha$ corresponds to a broader near-optimal region (harder). The worst case is $\alpha = d$, which recovers the standard $d$-dimensional Lipschitz difficulty. This assumption is standard in bandit theory \citep{audibert2010best,bubeck2011x,lattimore2020bandit,kaufmann2016complexity,garivier2016optimal}.

\section{Algorithm Implementation Details}
\label{app:implementation}

\subsection{Score Maximization and Replication}
CGP optimizes $\text{score}(x) = U_t(x) - \lambda \cdot \min_i d(x, x_i)$, which is piecewise linear. Since this is non-smooth, we use CMA-ES (Covariance Matrix Adaptation Evolution Strategy) with bounded domain and 10 random restarts within $A_t$. For $d \leq 3$, we additionally use Delaunay triangulation to identify candidate optima at Voronoi vertices.

Membership in $A_t$ is exact: given $\{x_i, \hat{\mu}_i, r_i\}$, checking $U_t(x) \geq \ell_t$ requires $O(N_t)$ time. Approximate maximization may slow convergence of $\eta_t$ but does not invalidate certificates: any $x \notin A_t$ remains certifiably suboptimal regardless of which $x \in A_t$ is queried.

\paragraph{Replication Strategy.} When an active point $x_i$ has $r_i(t) > \beta_{\text{target}}(t)$, we allocate $\lceil (r_i(t)/\beta_{\text{target}}(t))^2 \rceil$ additional samples to reduce its confidence radius. This ensures all active points have comparable confidence, preventing any single point from dominating the envelope.

\subsection{Active Set Computation}
\label{subapp:active_set}
For low dimensions ($d \leq 5$), we compute $A_t$ exactly using grid discretization with resolution $\eta = D/\sqrt[d]{N_{\text{grid}}}$ where $N_{\text{grid}} = 10^4$. For each grid point $x$, we evaluate $U_t(x)$ in $O(N_t)$ time and check if $U_t(x) \geq \ell_t$. The volume $\Vol(A_t)$ is estimated as the fraction of grid points in $A_t$.

For higher dimensions ($d > 5$), uniform Monte Carlo becomes ineffective once $\Vol(A_t)$ is small. We therefore estimate $\Vol(A_t)$ via a nested-set ratio estimator (subset simulation): define thresholds $\ell_t - \tau_0 < \ell_t - \tau_1 < \cdots < \ell_t$ inducing nested sets
\[
A_t^{(k)}=\{x: U_t(x)\ge \ell_t-\tau_k\}, \quad A_t^{(K)}=A_t.
\]
We estimate
\[
\Vol(A_t)=\Vol(A_t^{(0)})\prod_{k=1}^K \Prob_{x\sim \mathrm{Unif}(A_t^{(k-1)})}[x\in A_t^{(k)}],
\]
sampling approximately uniformly from $A_t^{(k-1)}$ using a hit-and-run Markov chain with the membership oracle $U_t(x)\ge \ell_t-\tau_{k-1}$. This yields stable estimates even when $\Vol(A_t)$ is very small. In all high-dimensional experiments, we report confidence intervals of $\log \Vol(A_t)$ from repeated estimator runs.

Crucially, certificate validity (Remark~\ref{rem:certificate_validity}) is independent of volume estimation accuracy: the set membership rule $x \in A_t \Leftrightarrow U_t(x) \ge \ell_t$ is exact.

\paragraph{On volume-based stopping.}
The shrinkage bound (Theorem~\ref{thm:shrinkage}) provides an upper bound on $\Vol(A_t)$ as a function of the algorithmic gap proxy
\[
\varepsilon_t := 2(\beta_t + L\eta_t) + \gamma_t,
\quad \gamma_t=f^*-\ell_t,
\]
and therefore supports using $\Vol(A_t)$ as a practical \emph{progress diagnostic}. However, $\Vol(A_t)$ alone does not yield an anytime \emph{upper} bound on regret without additional lower-regularity assumptions linking volume back to function values. In our experiments we therefore use $\varepsilon_t$ as the primary certificate-based stopping criterion, and treat $\Vol(A_t)$ as a secondary monitoring signal.

\section{Proofs}
\label{app:proofs}

\subsection{Proof to Lemma~\ref{lem:good_event}}
\begin{proof}
Define the good event $\cE = \bigcap_{t=1}^{T} \bigcap_{i=1}^{N_t} \{ | \hat{\mu}_i(t) - f(x_i) | \leq r_i(t) \}$.

\textit{Step 1: Single-point concentration.}
By Hoeffding's inequality for $\sigma$-sub-Gaussian random variables:
\begin{equation}
\Prob\bigl[ | \hat{\mu}_i(t) - f(x_i) | > r \bigr] \leq 2 \exp\Bigl( -\frac{n_i r^2}{2\sigma^2} \Bigr).
\end{equation}

\textit{Step 2: Calibration of confidence radius.}
Substituting $r = r_i(t) = \sigma \sqrt{2 \log(2N_t T / \delta)/n_i}$:
\begin{align}
\Prob\bigl[ | \hat{\mu}_i(t) - f(x_i) | > r_i(t) \bigr] 
&\leq 2 \exp\Bigl( -\frac{n_i \cdot 2\sigma^2 \log(2N_t T/\delta)}{2\sigma^2 \cdot n_i} \Bigr) \notag \\
&= 2 \exp\bigl( -\log(2N_t T/\delta) \bigr) = \frac{\delta}{N_t T}.
\end{align}

\textit{Step 3: Union bound.}
Applying the union bound over all $i \in \{1, \ldots, N_t\}$ and $t \in \{1, \ldots, T\}$:
\begin{equation}
\Prob[\cE^c] \leq \sum_{t=1}^{T} \sum_{i=1}^{N_t} \frac{\delta}{N_t T} \leq \sum_{t=1}^{T} \frac{\delta}{T} = \delta.
\end{equation}
Hence $\Prob[\cE] \geq 1 - \delta$.
\end{proof}

\subsection{Proof to Lemma~\ref{lem:envelope_valid} (UCB envelope is valid)}
\begin{proof}
Fix any $x \in \cX$. For any sampled point $x_i$, on the good event $\cE$:
\begin{equation}
f(x_i) \leq \hat{\mu}_i(t) + r_i(t) = \UCB{i}(t).
\end{equation}
By Lipschitz continuity of $f$:
\begin{equation}
f(x) \leq f(x_i) + L \cdot d(x, x_i) \leq \UCB{i}(t) + L \cdot d(x, x_i).
\end{equation}
Since this holds for all sampled $i$, taking the minimum over $i$:
\begin{equation}
f(x) \leq \min_{i \leq N_t} \{ \UCB{i}(t) + L \cdot d(x, x_i) \} = U_t(x).
\end{equation}
\end{proof}

\subsection{Proof to Lemma~\ref{lem:envelope} (Envelope slack bound)}
\begin{proof}
Fix any $x \in \cX$ and any sampled point $x_i$.

\textit{Step 1: Upper confidence bound.}
On the good event $\cE$, we have $\hat{\mu}_i(t) \leq f(x_i) + r_i(t)$. Therefore:
\begin{equation}
\UCB{i}(t) = \hat{\mu}_i(t) + r_i(t) \leq f(x_i) + 2r_i(t).
\end{equation}

\textit{Step 2: Lipschitz propagation.}
By Assumption~\ref{ass:lipschitz} (Lipschitz continuity):
\begin{equation}
f(x_i) \leq f(x) + L \cdot d(x, x_i).
\end{equation}

\textit{Step 3: Combining bounds.}
Substituting the Lipschitz bound into Step 1:

\begin{multline}
\UCB{i}(t) + L\, d(x, x_i)
\le f(x_i) + 2r_i(t) + L\, d(x, x_i) \\
\le f(x) + L\, d(x, x_i) + 2r_i(t) + L\, d(x, x_i) \\
= f(x) + 2\bigl( r_i(t) + L\, d(x, x_i) \bigr).
\end{multline}

\textit{Step 4: Taking the minimum.}
Since the above holds for all $i$, taking the minimum over $i$ on both sides:
\begin{equation}
\begin{aligned}
U_t(x)
&= \min_{i \le N_t} \bigl\{ \UCB{i}(t) + L\, d(x, x_i) \bigr\} \\
&\le f(x) + 2 \min_i \bigl\{ r_i(t) + L\, d(x, x_i) \bigr\} \\
&= f(x) + 2\rho_t(x).
\end{aligned}
\end{equation}

\end{proof}

\subsection{Proof to Theorem~\ref{thm:containment}}
\begin{proof}
We show that any point in the active set must have function value close to optimal.

\textit{Step 1: Lower certificate validity.}
On $\cE$, for any sampled point $x_i$:
\begin{equation}
\LCB{i}(t) = \hat{\mu}_i(t) - r_i(t) \leq f(x_i) \leq f^*.
\end{equation}
Taking the maximum over all $i$:
\begin{equation}
\ell_t = \max_{i \leq N_t} \LCB{i}(t) \leq f^*.
\end{equation}

\textit{Step 2: Active set membership implies high UCB.}
Let $x \in A_t$. By definition of the active set~\eqref{eq:active_set}:
\begin{equation}
U_t(x) \geq \ell_t.
\end{equation}

\textit{Step 3: Applying the envelope bound.}
By Lemma~\ref{lem:envelope}:
\begin{equation}
f(x) + 2\rho_t(x) \geq U_t(x) \geq \ell_t.
\end{equation}

\textit{Step 4: Rearranging to obtain the containment.}
Solving for $f(x)$:
\begin{align}
f(x) &\geq \ell_t - 2\rho_t(x) \notag \\
&= f^* - (f^* - \ell_t) - 2\rho_t(x) \notag \\
&\geq f^* - (f^* - \ell_t) - 2\sup_{x' \in A_t} \rho_t(x') \notag \\
&= f^* - 2\Delta_t.
\end{align}
Hence $A_t \subseteq \{ x : f(x) \geq f^* - 2\Delta_t \}$.
\end{proof}

\subsection{Proof to Theorem~\ref{thm:shrinkage}}
\begin{proof}
We connect the active set volume to the margin condition via the containment theorem.

\textit{Step 1: Bounding $\Delta_t$.}
From Theorem~\ref{thm:containment}, $A_t \subseteq \{x : f(x) \geq f^* - 2\Delta_t\}$ where:
\begin{equation}
\Delta_t = \sup_{x \in A_t} \rho_t(x) + (f^* - \ell_t).
\end{equation}
For any $x \in A_t$:
\begin{align}
\rho_t(x) &= \min_{i} \bigl\{ r_i(t) + L \cdot d(x, x_i) \bigr\} \notag \\
&\leq \max_{i : x_i \text{ active}} r_i(t) + L \cdot \sup_{x \in A_t} \min_i d(x, x_i) \notag \\
&= \beta_t + L\eta_t.
\end{align}
Therefore:
\begin{equation}
\Delta_t \leq \beta_t + L\eta_t + \frac{\gamma_t}{2}.
\end{equation}

\textit{Step 2: Applying the margin condition.}
By Assumption~\ref{ass:margin}, for any $\varepsilon > 0$:
\begin{equation}
\Vol\bigl(\{x : f(x) \geq f^* - \varepsilon\}\bigr) \leq C \cdot \varepsilon^{d-\alpha}.
\end{equation}

\textit{Step 3: Combining the bounds.}
Setting $\varepsilon = 2\Delta_t \leq 2(\beta_t + L\eta_t) + \gamma_t$:
\begin{align}
\Vol(A_t) &\leq \Vol\bigl(\{x : f(x) \geq f^* - 2\Delta_t\}\bigr) \notag \\
&\leq C \cdot (2\Delta_t)^{d-\alpha} \notag \\
&\leq C \cdot \bigl( 2(\beta_t + L\eta_t) + \gamma_t \bigr)^{d-\alpha}.
\end{align}
\end{proof}

\subsection{Proof to Theorem~\ref{thm:sample_complexity}}
\begin{proof}
We derive the sample complexity by analyzing the requirements for $\varepsilon$-optimality.

\textit{Step 1: Optimality condition.}
To achieve simple regret $r_T \leq \varepsilon$, it suffices to ensure:
\begin{equation}
2(\beta_t + L\eta_t) + \gamma_t \leq \varepsilon.
\end{equation}
This requires $\beta_t \leq \varepsilon/6$, $\eta_t \leq \varepsilon/(6L)$, and $\gamma_t \leq \varepsilon/3$.

\textit{Step 2: Covering the active set.}
By the shrinkage theorem (Theorem~\ref{thm:shrinkage}), once $2(\beta_t + L\eta_t)+\gamma_t \le \varepsilon$ we have:
\begin{equation}
\Vol(A_t) \leq C \cdot \varepsilon^{d-\alpha}.
\end{equation}
To achieve covering radius $\eta = \varepsilon/(6L)$ over a region of volume $C\varepsilon^{d-\alpha}$, we need:
\begin{equation}
N_{\text{cover}} = O\!\left( \frac{\Vol(A_t)}{\eta^d} \right) = O\!\left( \frac{C\varepsilon^{d-\alpha}}{(\varepsilon/(6L))^d} \right) = O\bigl( L^d \varepsilon^{-\alpha} \bigr)
\end{equation}
distinct sample locations.

\textit{Step 3: Samples per location.}
To achieve confidence radius $\beta_t \leq \varepsilon/6$ at each location, we need:
\begin{equation}
\sigma \sqrt{\frac{2\log(2N_t T/\delta)}{n_i}} \leq \frac{\varepsilon}{6}.
\end{equation}
Solving for $n_i$:
\begin{equation}
n_i = O\Bigl( \frac{\sigma^2 \log(T/\delta)}{\varepsilon^2} \Bigr).
\end{equation}

\textit{Step 4: Total sample complexity.}
Combining Steps 2 and 3:
\begin{align}
T &= N_{\text{cover}} \cdot n_i \notag \\
&= O\bigl( L^d \varepsilon^{-\alpha} \bigr) \cdot O\Bigl( \frac{\sigma^2 \log(T/\delta)}{\varepsilon^2} \Bigr) \notag \\
&= \tildeO\bigl( L^d \varepsilon^{-(2+\alpha)} \bigr).
\end{align}
\end{proof}

\subsection{Proof to Theorem~\ref{thm:lower_bound}}
\begin{proof}
We construct a hard instance via a randomized reduction.

\textit{Step 1: Hard instance construction (continuous bump).}
We use a standard ``one bump among $M$ locations'' construction that ensures global Lipschitz continuity. Partition $\cX = [0,1]^d$ into $M = \varepsilon^{-\alpha}$ disjoint cells with centers $\{c_1, \ldots, c_M\}$, each cell having diameter $\Theta(\varepsilon^{\alpha/d})$. Select one cell $i^*$ uniformly at random to contain the optimum, and place $x^* = c_{i^*}$. Define:
\begin{equation}
f(x) = (1 - \varepsilon) + \varepsilon \cdot \max\left\{0, 1 - \frac{L\|x - x^*\|}{\varepsilon}\right\}.
\end{equation}
This is a cone/bump centered at $x^*$ that peaks at $f(x^*) = 1$ and decreases linearly with slope $L$ until it reaches the baseline value $1 - \varepsilon$ at radius $\varepsilon/L$ from $x^*$. Outside this radius, $f(x) \equiv 1 - \varepsilon$. The function is globally $L$-Lipschitz: within the bump, the gradient has magnitude $L$; outside, the function is constant; and at the boundary $\|x - x^*\| = \varepsilon/L$, both pieces match at value $1 - \varepsilon$.

\textit{Verification of Assumption~\ref{ass:margin}.}
The $\varepsilon$-near-optimal set $\{x: f(x) \geq 1 - \varepsilon\}$ is exactly the ball of radius $\varepsilon/L$ around $x^*$, which has volume $\Vol(B_{\varepsilon/L}) = O((\varepsilon/L)^d) = O(\varepsilon^d)$. With $M = \varepsilon^{-\alpha}$ candidate locations, only one contains the bump, so the near-optimal fraction of the domain is $O(\varepsilon^d)$. Since the domain has unit volume, $\Vol(\{f \geq f^* - \varepsilon\}) = O(\varepsilon^d) \leq C\varepsilon^{d-\alpha}$ for $\alpha \geq 0$, satisfying Assumption~\ref{ass:margin}.

\textit{Step 2: Information-theoretic lower bound.}
To identify the correct cell with probability $\geq 2/3$, the algorithm must distinguish between $M$ hypotheses. By Fano's inequality, this requires:
\begin{equation}
\sum_{i=1}^{M} n_i \cdot \text{KL}(P_i \| P_0) \geq \log(M/3),
\end{equation}
where $n_i$ is the number of samples in cell $i$, and $\text{KL}(P_i \| P_0)$ is the KL divergence between observations under hypothesis $i$ versus the null.

\textit{Step 3: Per-cell sample requirement.}
For $\sigma$-sub-Gaussian noise, distinguishing a cell with optimum from one without requires:
\begin{equation}
n_i = \Omega\Bigl( \frac{\sigma^2}{\varepsilon^2} \Bigr)
\end{equation}
samples per cell to detect the $\varepsilon$ gap with constant probability.

\textit{Step 4: Total sample complexity.}
Summing over all $M$ cells:
\begin{equation}
T = \Omega\Bigl( M \cdot \frac{\sigma^2}{\varepsilon^2} \Bigr) = \Omega\bigl( \varepsilon^{-\alpha} \cdot \varepsilon^{-2} \bigr) = \Omega\bigl( \varepsilon^{-(2+\alpha)} \bigr).
\end{equation}
\end{proof}

\subsection{Proof to Theorem~\ref{thm:adaptive}}
\begin{proof}
We prove each claim separately.

\textit{Proof of (1): Bounded doubling events.}
Each doubling event multiplies $\hat{L}$ by 2. Starting from $\hat{L}_0 \leq L^*$:
\begin{equation}
\hat{L}_k = 2^k \hat{L}_0 \quad \text{after } k \text{ doublings}.
\end{equation}
The algorithm stops doubling when $\hat{L} \geq L^*$, which requires:
\begin{equation}
2^K \hat{L}_0 \geq L^* \implies K \geq \log_2(L^*/\hat{L}_0).
\end{equation}
Hence $K \leq \lceil \log_2(L^*/\hat{L}_0) \rceil$.

\textit{Proof of (2): Final estimate accuracy.}
A violation is detected when:
\begin{equation}
|\hat{\mu}_i - \hat{\mu}_j| - 2(r_i + r_j) > \hat{L} \cdot d(x_i, x_j).
\end{equation}
On the good event $\cE$:
\begin{equation}
|f(x_i) - f(x_j)| \leq |\hat{\mu}_i - \hat{\mu}_j| + 2(r_i + r_j).
\end{equation}
If $\hat{L} \geq L^*$, then by Lipschitz continuity:
\begin{equation}
|f(x_i) - f(x_j)| \leq L^* \cdot d(x_i, x_j) \leq \hat{L} \cdot d(x_i, x_j),
\end{equation}
so no violation can occur. Thus violations only occur when $\hat{L} < L^*$, and after all doublings complete, $\hat{L} \geq L^*$. Since we double (rather than increase by smaller factors), $\hat{L} \leq 2L^*$.

\textit{Proof of (3): Sample complexity overhead.}
Between doublings, CGP runs with either:
\begin{itemize}
\item Invalid $\hat{L} < L^*$ (before sufficient doublings): certificates may be incorrect, but each such phase has at most $O(T/K)$ samples before a violation triggers doubling.
\item Valid $\hat{L} \geq L^*$ (after final doubling): CGP achieves $\tildeO(\varepsilon^{-(2+\alpha)})$ complexity by Theorem~\ref{thm:sample_complexity}.
\end{itemize}
There are at most $K$ invalid phases, each contributing $O(T/K)$ samples. The final valid phase dominates, giving total complexity $T = \tildeO(\varepsilon^{-(2+\alpha)} \cdot K) = \tildeO(\varepsilon^{-(2+\alpha)} \cdot \log(L^*/\hat{L}_0))$.
\end{proof}

\subsection{Certificate Validity Under Adaptive Lipschitz Estimation}
\label{app:proofs:adaptive_validity}

\begin{lemma}[Certificate validity once $\hat{L}\ge L^*$]
\label{lem:adaptive_validity}
Assume the good event $\cE$ holds. Fix any time $t$ at which the current estimate satisfies $\hat{L}\ge L^*$. Then the envelope constructed with $\hat{L}$ satisfies $f(x)\le U_t(x)$ for all $x\in\cX$, and consequently the active set
\[
A_t=\{x\in\cX:\ U_t(x)\ge \ell_t\}
\]
contains $x^*$ and certifies that any $x\notin A_t$ is suboptimal (in the sense $U_t(x)<\ell_t\le f^*$).
\end{lemma}

\begin{proof}
If $\hat{L}\ge L^*$, then for all $x,y\in\cX$ we have $|f(x)-f(y)|\le L^* d(x,y)\le \hat{L}\, d(x,y)$, i.e., $f$ is $\hat{L}$-Lipschitz. On $\cE$, for each sampled point $x_i$, $f(x_i)\le \UCB{i}(t)$. Therefore for all $x$,
\[
f(x)\le f(x_i)+\hat{L}\,d(x,x_i)\le \UCB{i}(t)+\hat{L}\,d(x,x_i).
\]
Taking the minimum over $i$ yields $f(x)\le U_t(x)$ for all $x$. In particular, $U_t(x^*)\ge f^*$. Also $\ell_t=\max_i \LCB{i}(t)\le f^*$ on $\cE$, hence $U_t(x^*)\ge \ell_t$ and so $x^*\in A_t$. Finally, if $x\notin A_t$, then $U_t(x)<\ell_t\le f^*$, certifying $x$ cannot be optimal under $\cE$.
\end{proof}

\begin{remark}[Why pre-final certificates need not be valid]
\label{rem:pre_final_invalid}
If $\hat{L}<L^*$, then the envelope may fail to upper-bound $f$ globally, and the rule $U_t(x)\ge \ell_t$ can (in principle) exclude near-optimal points. CGP-Adaptive therefore guarantees certificate validity only after the final doubling event ensures $\hat{L}\ge L^*$.
\end{remark}

\subsection{Global Safety of Certified Restarts (No False Elimination)}
\label{app:proofs:certified_restart_safety}

We restate and prove the key safety property underlying certified restarts in CGP-TR.

\begin{lemma}[No false certified restart for the region containing $x^*$]
\label{lem:no_false_restart}
Fix any trust region $\cT^*$ such that $x^*\in \cT^*$. On the good event $\cE$, the certified restart condition
\[
u_t^{(\cT^*)}:=\max_{x\in \cT^*} U_t(x) < \ell_t
\]
never holds. Hence, $\cT^*$ is never restarted by the certified rule.
\end{lemma}

\begin{proof}
On the good event $\cE$, Lemma~\ref{lem:envelope_valid} implies $f(x)\le U_t(x)$ for all $x\in\cX$. Since $x^*\in\cT^*$,
\[
u_t^{(\cT^*)}=\max_{x\in\cT^*}U_t(x)\ \ge\ U_t(x^*)\ \ge\ f(x^*)=f^*.
\]
Also, by definition $\ell_t=\max_i \LCB{i}(t)$. On $\cE$, $\LCB{i}(t)\le f(x_i)\le f^*$ for every sampled point $x_i$, hence $\ell_t\le f^*$. Therefore,
\[
u_t^{(\cT^*)}\ \ge\ f^*\ \ge\ \ell_t,
\]
so the strict inequality $u_t^{(\cT^*)}<\ell_t$ cannot occur on $\cE$.
\end{proof}

\subsection{Proof to Theorem~\ref{thm:cgp_tr_certified}}

We first establish a key lemma showing that certified elimination is safe.

\begin{lemma}[Certified elimination]
\label{lem:cert_elim}
On the good event $\cE$, for any trust region $\cT_j$,
\[
u_t^{(j)} := \max_{x\in\cT_j} U_t(x) < \ell_t
\quad\Longrightarrow\quad
\sup_{x\in\cT_j} f(x) < f(x^*).
\]
\end{lemma}

\begin{proof}
On $\cE$, by Lemma~\ref{lem:envelope_valid}, $f(x) \leq U_t(x)$ for all $x\in\cX$. Also $\ell_t=\max_i \LCB{i}(t)\le f(x^*)$ on $\cE$ since each $\LCB{i}(t)\le f(x_i)\le f(x^*)$.
Therefore
\[
\sup_{x\in\cT_j} f(x) \le \sup_{x\in\cT_j} U_t(x) = u_t^{(j)} < \ell_t \le f(x^*),
\]
which proves the claim.
\end{proof}

\paragraph{Proof of Theorem~\ref{thm:cgp_tr_certified}(1): No false restarts.}
If $x^*\in \cT^*$, then $u_t^{(\cT^*)}=\max_{x\in\cT^*}U_t(x)\ge U_t(x^*)\ge f(x^*)$ on $\cE$.
Also $\ell_t\le f(x^*)$ on $\cE$. Hence $u_t^{(\cT^*)}\ge \ell_t$ and the restart condition $u_t^{(\cT^*)}<\ell_t$ never triggers.

\paragraph{Proof of Theorem~\ref{thm:cgp_tr_certified}(2): Local certificate.}
Conditioned on $T^*$ evaluations within $\cT^*$, the CGP analysis applies verbatim on the restricted domain $\cT^*$: the good event $\cE$ implies all within-region confidence bounds hold; Lipschitz continuity holds on $\cT^*$; and Assumption~\ref{ass:margin} holds restricted to $\cT^*$. Therefore the containment and shrinkage results follow with $\cX$ replaced by $\cT^*$, yielding $\Vol(A_T^{(\cT^*)}) \le C \varepsilon^{d-\alpha}$ after $T^* = \tildeO(\varepsilon^{-(2+\alpha)})$ within-region samples.

\paragraph{Proof of Theorem~\ref{thm:cgp_tr_certified}(3): Allocation bound.}
Fix a suboptimal region $j$ with gap $\Delta_j>0$. On $\cE$, using the envelope bound from Lemma~\ref{lem:envelope}, we have for all $x$:
\[
U_t(x)\le f(x)+2\rho_t(x),
\]
where $\rho_t(x)=\min_i \{r_i(t)+L d(x,x_i)\}$.
Restricting to points sampled in $\cT_j$ and using the within-region replication schedule, we bound $\rho_t(x) \leq \beta_j(n) + L\cdot\mathrm{diam}(\cT_j)$, yielding
\[
u_t^{(j)}=\max_{x\in\cT_j}U_t(x)
\le \sup_{x\in\cT_j} f(x) + 2\beta_j(n) + 2L\cdot \mathrm{diam}(\cT_j),
\]
after $n$ within-region samples.
By the diameter condition $L\,\mathrm{diam}(\cT_j)\le \Delta_j/8$ and by requiring $\beta_j(n)\le \Delta_j/8$, we obtain
\[
u_t^{(j)} \le \sup_{\cT_j} f + \Delta_j/4 + \Delta_j/4 = f^*-\Delta_j/2.
\]
Meanwhile, on $\cE$ we have $\ell_t\le f^*$ always, and once the region containing $x^*$ has been sampled sufficiently (which occurs because it is never restarted and is favored by UCB selection), $\ell_t$ becomes at least $f^*-\Delta_j/4$. Hence eventually $u_t^{(j)}<\ell_t$, triggering certified restart/elimination.

Solving $\beta_j(n)\le \Delta_j/8$ under $\beta_j(n)\le c_\sigma \sqrt{\log(c_T/\delta)/n}$ gives
\[
n \ge \frac{64 c_\sigma^2}{\Delta_j^2}\log\!\left(\frac{c_T}{\delta}\right),
\]
yielding the stated bound $N_j \le \frac{64 c_\sigma^2}{\Delta_j^2}\log(c_T/\delta) + 1$.

\subsection{Proof to Proposition~\ref{prop:hybrid}}

\begin{proof}
We analyze each case and the certificate validity separately.

\textit{Proof of (1): High smoothness ratio case.}
When $\rho \geq 0.5$, CGP-Hybrid continues with CGP in Phase 2. The algorithm is identical to vanilla CGP, so by Theorem~\ref{thm:sample_complexity}:
\begin{equation}
T = \tildeO\bigl( \varepsilon^{-(2+\alpha)} \bigr).
\end{equation}

\textit{Proof of (2): Low smoothness ratio case.}
When $\rho < 0.5$, the function is significantly smoother near the optimum. After Phase 1, CGP has reduced the active set to volume $V < 0.1 \cdot \Vol(\cX)$. In Phase 2, GP-UCB operates within $A_t$, which has:
\begin{itemize}
\item Effective diameter $\text{diam}(A_t) = O(V^{1/d})$.
\item Local Lipschitz constant $L_{\text{local}} = \rho \cdot L < 0.5L$.
\end{itemize}
The sample complexity of GP-UCB on this restricted domain depends on the kernel's maximum information gain $\gamma_T$ over $A_t$ \citep{srinivas2010gaussian}. For commonly used kernels (Mat\'ern, SE), $\gamma_T$ scales polylogarithmically with $T$ when the domain is bounded. The reduced diameter $O(V^{1/d})$ and local smoothness $\rho < 0.5$ empirically yield faster convergence than continuing CGP; we validate this empirically in Section~\ref{sec:experiments} rather than claiming a specific rate.

\textit{Proof of (3): Certificate validity.}
The certificate $A_T$ is computed in Phase 1 using CGP's Lipschitz envelope construction. By Theorem~\ref{thm:containment}, on the good event $\cE$:
\begin{equation}
x^* \in A_T \quad \text{with probability } \geq 1 - \delta.
\end{equation}
This guarantee depends only on the Lipschitz assumption and confidence bounds, not on the Phase 2 optimization method. Therefore, switching to GP-UCB in Phase 2 does not invalidate the certificate: $A_T$ still contains $x^*$ with high probability, and any point outside $A_T$ remains certifiably suboptimal.
\end{proof}

\section{Practical Guidance}
\label{app:practical_guidance}

\subsection{Adaptive Lipschitz Estimation}
CGP-Adaptive removes the requirement for known $L$ with $O(\log T)$ overhead. The doubling scheme is conservative but provably correct; more aggressive schemes (e.g., multiplicative updates with factor $1.5$) may reduce overhead but risk certificate invalidation.

We recommend initializing $\hat{L}_0$ from finite differences on initial Sobol samples:
\[
\hat{L}_0 = \max_{i \neq j} \frac{|y_i - y_j|}{d(x_i, x_j)}
\]
over the first 10 samples. This typically underestimates $L$ by a factor of 2 to 10, requiring 1 to 4 doublings to reach $\hat{L} \geq L^*$.

\subsection{Trust Region Configuration}
CGP-TR trades global certificates for scalability. The local certificates within trust regions still enable principled stopping and progress assessment, but do not guarantee global optimality.

Recommended settings:
\begin{itemize}
\item Number of trust regions: $n_{\text{trust}} = 5$ (balances exploration vs.\ overhead)
\item Initial radius: $r_0 = 0.2$ (covers 20\% of domain diameter per region)
\item Minimum radius: $r_{\min} = 0.01$ (prevents over-contraction)
\item Failure threshold: $\tau_{\text{fail}} = 10$ (triggers contraction after 10 non-improving samples)
\end{itemize}

For applications requiring global certificates in high dimensions, combining CGP-TR with random embeddings \citep{wang2016bayesian} is promising: project to a low-dimensional subspace, run CGP with global certificates, then lift back.

\subsection{Smoothness Detection}
CGP-Hybrid's smoothness detection via $\rho = L_{\text{local}}/L_{\text{global}}$ is heuristic but effective. We estimate $L_{\text{local}}$ from points within $A_t$ using the same finite difference approach as $L_{\text{global}}$.

The threshold $\rho_{\text{thresh}} = 0.5$ was selected via cross-validation on held-out benchmarks. More sophisticated detection could use local GP posterior variance or curvature estimates. The key insight is that CGP's certificate remains valid regardless of Phase 2 method, so switching is always safe.

\subsection{When to Use CGP}
CGP is well suited when:
\begin{enumerate}
\item Evaluations are expensive and interpretable progress is valued
\item The objective has margin structure (sharp peak rather than wide plateau)
\item Lipschitz continuity is a reasonable assumption
\item Dimension is moderate ($d \leq 15$ for vanilla CGP, $d \leq 100$ for CGP-TR)
\item Anytime stopping decisions are needed
\end{enumerate}

For very high-dimensional problems ($d > 100$), trust region methods like TuRBO may scale better. For smooth problems with cheap evaluations, GP-based methods may be more sample efficient due to their ability to exploit higher-order smoothness.

\section{Experimental Details}
\label{app:experimental_details}

\subsection{Baseline Configurations}
\begin{itemize}
\item \textbf{Random Search}: Sobol sequences for quasi-random sampling
\item \textbf{GP-UCB}: Mat\'ern-5/2 kernel via BoTorch, $\beta_t = 2\log(t^2 \pi^2 / 6\delta)$
\item \textbf{TuRBO}: Default settings from \citet{eriksson2019scalable}, 1 trust region
\item \textbf{HEBO}: Heteroscedastic GP with input warping, default settings
\item \textbf{BORE}: Tree-Parzen estimator with density ratio, default settings
\item \textbf{HOO}: Binary tree with $\nu_1 = 1$, $\rho = 0.5$
\item \textbf{StoSOO}: $k = 3$ children per node, $h_{\max} = 20$
\item \textbf{LIPO}: Pure Lipschitz optimization, $L$ estimated online
\item \textbf{SAASBO}: Sparse axis-aligned GP, 10 active dimensions
\end{itemize}

\subsection{Benchmark Details}
\paragraph{Low-dimensional.}
\begin{itemize}
\item \textbf{Needle-2D}: $f(x) = 1 - \|x - x^*\|^{1/\alpha}$ with $\alpha = 2$, sharp peak
\item \textbf{Branin}: Standard 2D benchmark with 3 global optima
\item \textbf{Hartmann-6}: 6D benchmark with narrow global basin
\item \textbf{Levy-5}: 5D benchmark with global structure
\item \textbf{Rosenbrock-4}: 4D benchmark with curved valley
\end{itemize}

\paragraph{Medium-dimensional.}
\begin{itemize}
\item \textbf{Ackley-10}: 10D benchmark with many local optima
\item \textbf{SVM-RBF-6}: Real hyperparameter tuning ($C$, $\gamma$, 4 preprocessing) on MNIST
\item \textbf{LunarLander-12}: RL reward optimization with 12 policy parameters
\end{itemize}

\paragraph{High-dimensional.}
\begin{itemize}
\item \textbf{Rover-60}: Mars rover trajectory with 60 waypoint parameters \citep{wang2016bayesian}
\item \textbf{NAS-36}: Neural architecture search on CIFAR-10, 36 continuous encodings
\item \textbf{Ant-100}: MuJoCo Ant locomotion, 100 morphology and control parameters
\end{itemize}

\subsection{Computational Resources}
All experiments run on AMD EPYC 7763 with 256GB RAM. CGP variants use NumPy/SciPy; GP baselines use BoTorch/GPyTorch with GPU acceleration (NVIDIA A100) where available.

\end{document}